\documentclass[letterpaper, 10 pt, conference]{ieeeconf}  % Comment this line out if you need a4paper

\IEEEoverridecommandlockouts                              % This command is only needed if 
\usepackage{graphicx} % for pdf, bitmapped graphics files
\usepackage{amsmath} % assumes amsmath package installed
\usepackage{amssymb}  % assumes amsmath package installed
\usepackage{siunitx}
\usepackage{dblfloatfix}
\usepackage{subcaption}
\usepackage{verbatim}
\usepackage{algpseudocode}
\usepackage{algorithm}
\usepackage{caption}
\usepackage[font=small]{caption}
\usepackage{cite}
\usepackage{xcolor}
\usepackage{url}
\usepackage{lipsum, color}
\usepackage{balance} %equalize last page
\newtheorem{prop}{Proposition}
\DeclareMathOperator{\Var}{Var}

\title{\LARGE \bf
Reinforcement Learning with Non-uniform State Representations \\for Adaptive Search
}
 
\let\vec\boldvec %bold vectors

\newcommand{\robotloc}{\vec{x}}
\newcommand{\targetloc}{\vec{y}}
\newcommand{\state}{\vec{s}}
\DeclareMathOperator{\rmap}{q}

\author{Sandeep Manjanna$^{1}$, Herke van Hoof$^{2}$ and Gregory Dudek$^{1}$% <-this % stops a space
\thanks{$^{1}$Sandeep Manjanna and Gregory Dudek are with Mobile Robotics Lab (MRL), Center for Intelligent Machines, McGill University, Montreal, Canada
        {\tt\small (msandeep, dudek)@cim.mcgill.ca}}%
\thanks{$^{2}$Herke van Hoof performed this work while at McGill, but is now with the Amsterdam Machine Learning Lab (AMLAB), University of Amsterdam, Amsterdam, Netherlands
        {\tt\small h.c.vanhoof@uva.nl}}%
\thanks{\textbf{~\copyright2018~IEEE. Personal use of this material is permitted. Permission from IEEE must be obtained for all other uses, in any current or future media, including reprinting/republishing this material for advertising or promotional purposes, creating new collective works, for resale or redistribution to servers or lists, or reuse of any copyrighted component of this work in other works. DOI: 10.1109/SSRR.2018.8468649}}}
\begin{document}

\maketitle
\thispagestyle{empty}
\pagestyle{empty}

%%%%%%%%%%%%%%%%%%%%%%%%%%%%%%%%%%%%%%%%%%%%%%%%%%%%%%%%%%%%%%%%%%%%%%%%%%%%%%%%
\begin{abstract}

Efficient spatial exploration is a key aspect of search and rescue. 
%A complete scanning technique like a raster scan or lawnmower pattern is not practical if the environment is large and is inefficient in finding the lost target quickly. 
In this paper, we present a search algorithm that generates efficient  trajectories that optimize the rate at which probability mass is covered by a searcher. This should allow an autonomous vehicle find one or more lost targets as rapidly as possible.  We do this by performing non-uniform sampling of the search region. The path generated minimizes the expected time to locate the missing target by visiting high  probability regions using non-myopic path generation based on reinforcement learning. We model the target probability distribution using  a classic \emph{mixture of Gaussians} model with means and mixture coefficients tuned according to the location and time of sightings of the lost target.
%The estimated probability distribution  is used as an input for planning the search paths.
Key features of our search algorithm are the ability to employ a very general non-deterministic action model and the ability to generate action plans for any new probability distribution using the parameters learned on other similar looking distributions. One of the key contributions of this paper is the use of non-uniform state aggregation for policy search in the context of robotics.
%, i.e., an action planner trained on generic search and rescue distribution models is capable of planning high rewarding paths with a new spatial probability distribution map on the fly without being trained again.

We compare the paths generated by our algorithm with other 
accepted spatial coverage techniques such as distribution independent boustrophedonic coverage and model dependent spiral search. %We evaluate our algorithm on the efficiency of the paths generated in terms of probability mass reduction discounted with time.
%This discounting aids in covering high-probability regions earlier, thus pushing to locate the target with minimum time. 
We present a proof showing that rewarding for clearing probability mass instead of locating the target does not bias the objective function. The experiments show that the learned policy outperforms several well-known baselines even in scenarios different from the one it has been trained on.

\end{abstract}

%%%%%%%%%%%%%%%%%%%%%%%%%%%%%%%%%%%%%%%%%%%%%%%%%%%%%%%%%%%%%%%%%%%%%%%%%%%%%%%%
\section{INTRODUCTION}\label{sec:intro}

This paper addresses the discovery and synthesis of efficient search trajectories based on learned domain-specific policies and probabilistic inference. We also explore the use of non-uniform state space representations to enhance the performance of policy search methods.

Searching for a lost target or a person in the wilderness can be tedious, labor-intensive, imprecise, and very expensive. In some cases, the manual search and rescue operations are also unsafe for the humans involved. These reasons motivate the use of autonomous vehicles for such missions. In many search and rescue problems, timeliness is crucial. As time increases, the probability of success decreases considerably~\cite{pfau2013wilderness}, and every hour the effective search radius increases by approximately $3$ km \cite{setnicka1980wilderness}. We present an active sampling algorithm that generates an efficient path in real-time to search for a lost target, thus making an autonomous search and rescue mission realistic and more beneficial.

\begin{figure}[h]
\centering
\captionsetup{justification=centering}
\includegraphics[width=0.48\textwidth]{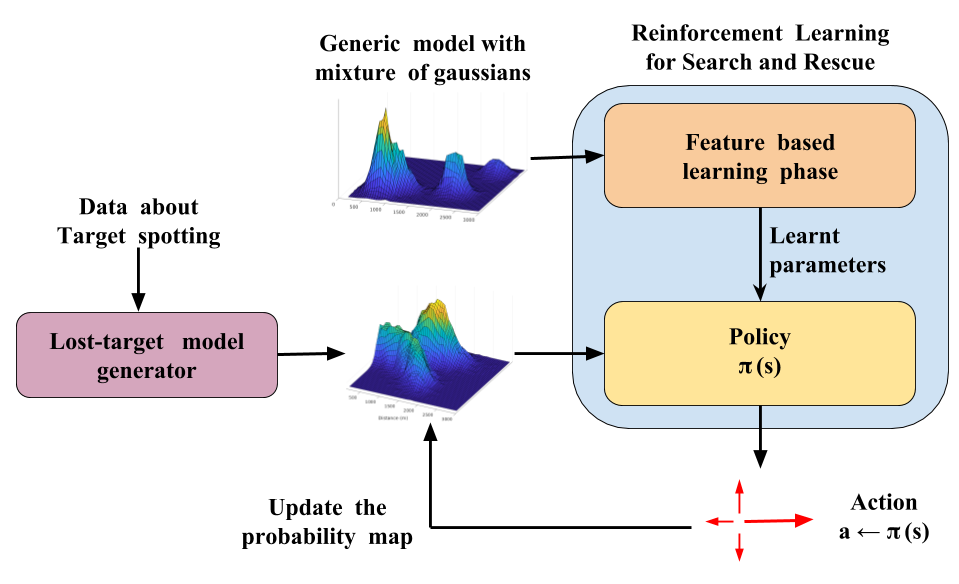}
\caption{Overview of our search and rescue approach.}
\vspace{-2.0em}
\label{fig:overview}
\end{figure}

Given an approximate region of interest, the first task is to decide where to search. Within the region of interest, some areas have higher probability of containing the target than the others. Data shows that more than $90\%$ of searches based on a probability distribution over the region have been resolved successfully within a short duration \cite{koester2008lost}. We model the spatial field of search using classic Gaussian mixture models based on sightings of the lost target. The modeled probability distribution map over the search region is used as an input for planning the search paths. The traditional way to assign probabilities to search regions is with a distance ring model: a simple ‘bulls-eye’ formed by the $25\%$, $50\%$, $75\%$, and $95\%$ probability circles \cite{koester2008lost}. There are more advanced methods that propose a Bayesian model that utilizes publicly available terrain data to help model lost person behaviors enabling domain experts to encode uncertainty in their prior estimations \cite{lin2010bayesian}. Since modeling of the search region probability distribution is not the focus of this paper, we use a simpler approach depending only on location and time of the target sightings. Our search algorithm would perform efficiently independent of the underlying probability distribution model.

Once the probability distribution of the search region is mapped, the task is to search this region efficiently such that the regions with high probability of finding the target (\emph{hotspots}) are searched in the early phase of the exploration. Many complete coverage approaches use line-sweep based techniques \cite{huang2001optimal,xu2011optimal,berger2016area} which are not suitable for search and rescue tasks because of their non-timeliness. 
%Also it has been observed that for low-pass multi-band signals, uniform sampling can be inefficient and sampling rates far below the Nyquist rate can still be information preserving \cite{venkataramani2000perfect, landau1967necessary}. This is the key guiding principle behind active and non-uniform sampling \cite{low2008adaptive,rahimi2005adaptive,sadat2015fractal}.
%When the underlying probability distributions are smoothly varying without any local-maxima peaks, non-adaptive strategies are known to perform well \cite{singh2006active}. However, in the case of search and rescue, the environment could contain peaks with high local-variance. Thus adaptive sampling can exploit the clustering phenomena to map the environmental field more accurately than non-adaptive sampling. 
Cooper \emph{et al.} discuss the relationship between area coverage and search effort density. They claim that any solution to the optimal search density problem is also a solution to the optimal coverage problem~\cite{cooper2003compatibility}. We present a proof showing that coverage of the search region to reduce probability mass is equivalent to searching for the target. We have previously demonstrated an algorithm to selectively cover a region based on an underlying reward distribution \cite{sandeep_crv_2016,sandeep_iros2017}. This technique, however, does not scale up smoothly for larger regions because of the computational complexity.

In this paper, we present an adaptive sampling technique that generates efficient paths to search the missing target as fast as possible by performing non-uniform sampling of the search region. The path generated minimizes the expected time to locate the missing target by visiting high search probability regions using non-myopic path planning based on reinforcement learning. A non-myopic plan is one that accounts for possible observations that can be made in the future~\cite{singh2009nonmyopic}. Fig. \ref{fig:overview} presents an overview of the whole path generation system. Our algorithm gets trained with a generic model of the probability distribution, which can be a map generated by a Gaussian mixture. These learnt parameters are used on the generated lost target probability distribution map to come up with an action plan (\emph{policy} $\pi$). For a given state, an action is chosen according to the probability distribution $\pi$ and thus a path is generated for the searcher robot. Training the system with discounted rewards helps the planner to achieve paths that cover hotspots at the earlier stages of the search task. A major contribution of this paper is the use of non-uniform state aggregation for policy search in the context of robotics.

The key feature of our search algorithm is the ability to generate action plans for any new probability distribution using the parameters learnt on other similar looking distributions, i.e. an action planner trained on generic search and rescue distribution models is capable of planning high rewarding paths with a new probability distribution map without being trained again.

\section{PROBLEM FORMULATION}\label{sec:formulation}

The search region is a continuous two-dimensional area of interest $\mathcal{E} \subset \mathbb{R}^2$ with user-defined boundaries. The spatial search region is discretized into uniform grid cells, such that the robot's position $\robotloc$ and the target's position $\targetloc$ can be represented by two pairs of integers $\robotloc, \targetloc \in \mathbb{Z}^2$. Each grid-cell $(i,j)$ is assigned a prior probability value $\rmap(i,j)$ of the target being present in that cell. 

The aim is to find the target in as short a time as possible. Formally, we specify this objective as maximizing $\mathbb{E}[\exp(-T/c)]$, where $T$ is the time elapsed until the target is found and $c$ is a condition-dependent constant. This objective reflects the assumption that the probability of the target's condition becoming bad is constant, and the aim is to locate the target in good condition. 

We will specify the robot's behavior using a parametrized policy. This is a conditional probability distribution $\pi_{\vec{\theta}}(\vec{s},\vec{a}) = p(\vec{a}|\vec{s};\vec{\theta})$ that maps a description of the current state $\vec{s}$ of the search to a distribution over possible \emph{actions} $\vec{a}$. Our aim will be to automatically find good parameters $\vec{\theta}$, after which the policy can be deployed without additional training on new problems. The maximization objective then becomes:
\begin{equation}
\max_\theta J(\theta)=\max_\theta \mathbb{E}_{\tau_{\vec{\theta}}}[\exp(-T/c)],
\label{eq:objective}
\end{equation}
where the expectation is taken with respect to trajectories $\tau_{\vec{\theta}}=(s_0,a_0,s_1,a_1,...)$ generated by the policy $\pi_{\vec{\theta}}$. 
%The search region is a continuous two-dimensional area of interest $\mathcal{E} \subset \mathbb{R}^2$ with user-defined boundaries. The spatial search region is discretized into uniform grid cells ($x_{ij} \in X_{map}$). Each grid-cell is assigned a utility value equivalent to the integral of underlying probability distribution over that cell. These grid-wise utility values constitute the rewards ($r_{ij} \in R_{map}$) received when a particular grid-cell is visited. Higher the probability of finding the target in a grid-cell, greater is the reward gained by visiting that cell. The map of all the grid-wise rewards is referred to as the rewardmap ($R_{map}$). Once a grid-cell $x_{ij}$ is visited, the associated reward $r_{ij}$ is collected and assigned with zero.

%The goal for efficient search with fixed time-steps ($T$) would be to maximize the total accumulated reward $\sum_{t=0}^{t=T-1} r_{ij}$. Additionally, we are interested in collecting higher rewards as early as possible. Hence, we want to maximize the total discounted reward $\sum_{t=0}^{t=T-1} \gamma^t r_{ij}$, where $\gamma$ is a discount factor that discounts rewards further in the future.

\section{MODELING OF THE SEARCH AREA}\label{sec:model}

We model the spatial field of search using a classic Gaussians mixture model. 
%with means and mixture coefficients tuned according to the location and time of sightings of the lost target. 
The map $\rmap$ is initialized according to these generated prior probabilities. An example map is illustrated in Fig.~\ref{fig:model}. Our search algorithm can take any map $\rmap$ in accordance to the prior belief over where the target is. In this work, we will test our method on randomly generated Gaussian mixtures, but the algorithm could be trained on any type of distribution.

% We use Gaussian Processes (GP)~\cite{rasmussenGP} to model the spatial field. In particular, a phenomenon over locations $\mathbf{W}$ can be estimated as a posterior distribution $p(f(\mathbf{W})\mid \mathbf{W},\mathbf{X}, \mathbf{Y}) \sim \mathcal{N}(\mathbf{\mu}_{\mathbf{W}}, \mathbf{\Sigma}_{\mathbf{W}})$ fitted over a set of noisy observations $\mathbf{Y}$ made at locations $\mathbf{X}$.
% The mean vector $\mathbf{\mu}_{\mathbf{W}}$ is obtained as $\mathbf{\mu}_{\mathbf{W}} = \mu(\mathbf{W}) + K(\mathbf{X},\mathbf{W})^T K(\mathbf{X},\mathbf{X})^{-1}(\mathbf{Y}-\mu(\mathbf{X}))$ and represents the estimate of the phenomenon, while the covariance matrix is given by $\mathbf{\Sigma}_{\mathbf{W}} = K(\mathbf{W},\mathbf{W}) - K(\mathbf{X},\mathbf{W})^T K(\mathbf{X},\mathbf{X})^{-1} K(\mathbf{X},\mathbf{W})$. 
% Mean and covariance functions should be formulated to completely define a GP. As done in mainstream approach, mean is assumed to be zero, and the covariance function $k(\mathbf{x}, \mathbf{x}')$, is a radial basis kernel (RBF):
% \begin{equation}
% \label{eq:gp_kernel}
% k(\mathbf{x}, \mathbf{x}') = \sigma_f^2 \exp\Big(-\frac{|\mathbf{x} - \mathbf{x}' |^2}{2\mathfrak{l}^2}\Big),
% \end{equation}
% \noindent where signal variance $\sigma_f^2$ and length scale $\mathfrak{l}^2$ are hyper-parameters that encode amplitude and smoothness. Note that, with a GP, it is possible to quantify the uncertainty of the estimates in $\mathbf{W}$ by looking at the main diagonal of $\mathbf{\Sigma}_{\mathbf{W}}$, also called \emph{predictive variance}. 
\vspace{-1.0em}
\begin{figure}[h]
\centering
\captionsetup{justification=centering}
\includegraphics[width=0.45\textwidth, height=5cm]{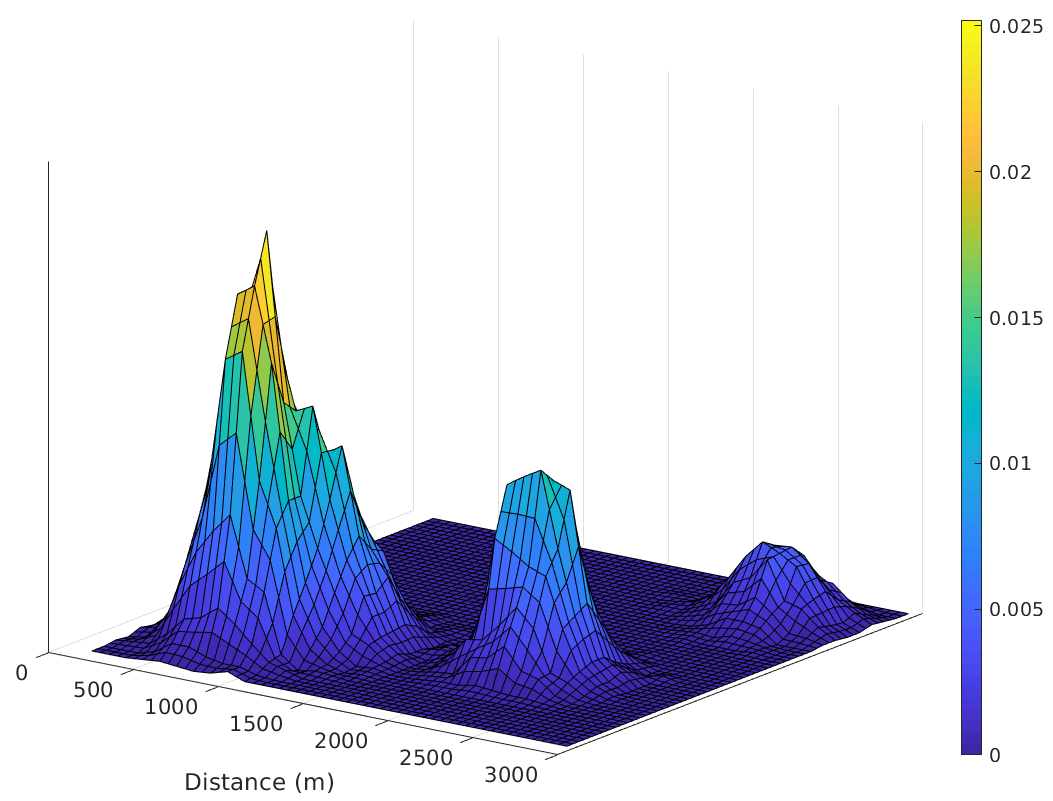}
\caption{Probability distribution modeled with \\Gaussian mixtures.}
\vspace{-1.0em}
\label{fig:model}
\end{figure}

\section{POLICY GRADIENT BASED SEARCH}
Finding a sequence of actions that maximizes a long-term objective could be done using dynamic programming. However, in our formulation the system state is described using a map containing the per-cell probability of the target being present and this map changes as areas are visited by the agent. The result is an extremely large state space where dynamic programming is impracticable - especially if the time to solve each particular problem is limited. 

Instead, we turn to methods that directly optimize the policy parameters $\vec{\theta}$ based on (simulated) experience. To apply these methods, we will first formalize the search problem as a Markov Decision Process (MDP).

\begin{figure*}[hb]
\centering
	\begin{subfigure}[b]{0.24\textwidth}
		\captionsetup{justification=centering}
        \includegraphics[width=\textwidth,height=4cm]{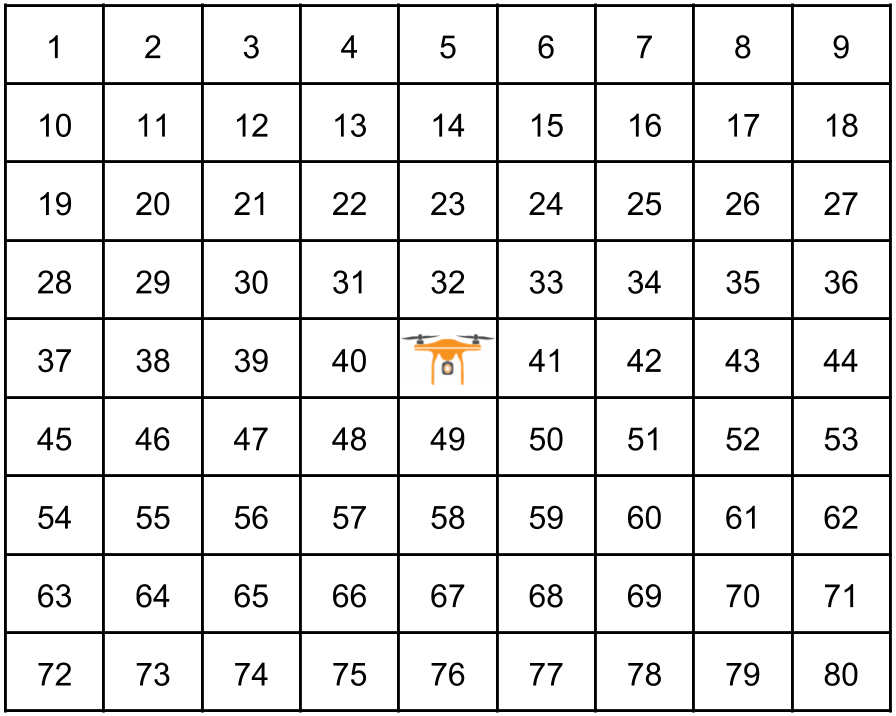}
        \caption{Feature design 1 - centered around the robot with a map of size $9\times9$. This is all-grid uniform feature design.}
        \label{fig:features1}
    \end{subfigure}
    \begin{subfigure}[b]{0.24\textwidth}
		\captionsetup{justification=centering}
        \includegraphics[width=\textwidth,height=4cm]{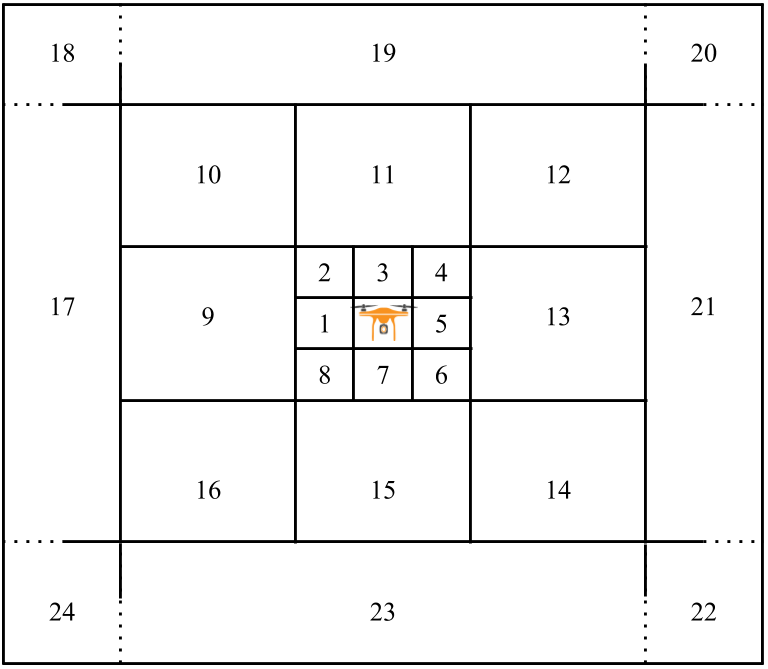}
        \caption{Feature design 2 - centered around the robot with a map of size $n\times n$, with $n \ge 10$. This is a multi-resolution feature design.}
        \label{fig:features2}
    \end{subfigure}
	\begin{subfigure}[b]{0.48\textwidth}
		\captionsetup{justification=centering}
        \includegraphics[width=\textwidth]{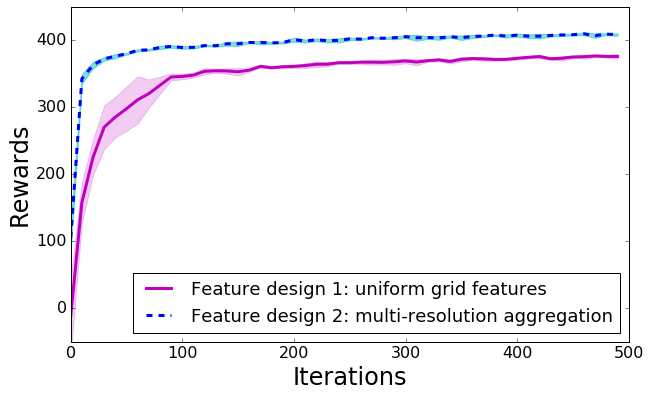}
        \caption{Average accumulated rewards with random starting locations}
        \label{fig:total}
    \end{subfigure}
    \begin{subfigure}[t]{0.48\textwidth}
       \captionsetup{justification=centering}
        \includegraphics[width=\textwidth]{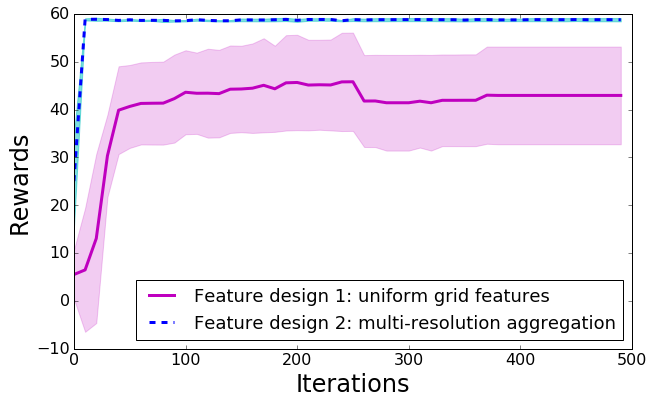}
        \caption{Average discounted rewards with a fixed starting location}
        \label{fig:discounted}
    \end{subfigure}
    \begin{subfigure}[t]{0.48\textwidth}
    \captionsetup{justification=centering}
        \includegraphics[width=\textwidth]{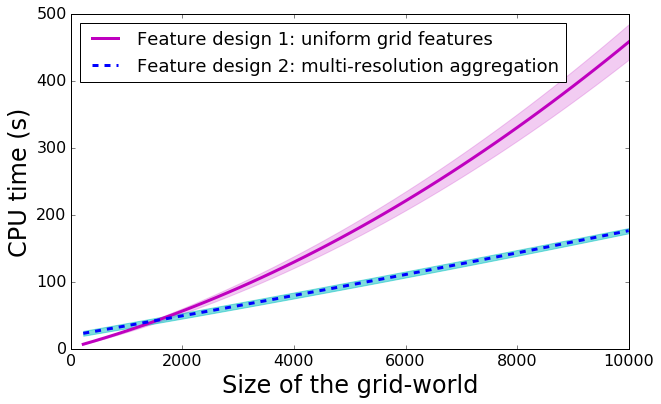}
        \caption{CPU-time to obtain a path vs. growth in the size of the search world}
        \label{fig:cpu_time}
    \end{subfigure}
    \captionsetup{justification=centering}
    \caption{Results from evaluation of two different kinds of feature designs. \\ The shaded region indicates the standard deviation over five trials on three different sized worlds.}
    \label{fig:feat_results}
    \vspace{-1em}
\end{figure*}

\subsection{Formalizing search as MDP}
A Markov Decision Process is a formal definition of a decision problem that is represented as a tuple $(S, A,T(\vec{s}_{t+1}|\vec{s}_t, \vec{a_t}),r(\vec{s}_t, \vec{a}_t),\gamma)$, where $S$ and $A$ are the state and action space, $T$ models transition dynamics based on the current state and action, and $r$ defines the reward for the current state-action pair. $\gamma$ is a discount factor that reduces the desirability of obtaining a reward $t$ time-steps from now rather than in the present by $\gamma^t$. The objective is then to optimize the expected discounted cumulative reward $J=\mathbb{E}_\tau[\sum_{t=0}^{H} \gamma^{t} r_t(s_t,a_t)]$, where $H$ is the optimization horizon.

In the proposed approach, we take the state $\state$ to include the position of the robot $\robotloc$ as well as the map $\rmap$ containing the per-location presence probability for the target, $\state=(\robotloc,\rmap)$. The actions we consider are for the robot to move to and scan the cell North, East, South or West of its current location. Transitions deterministically move the agent in the desired direction. When scanning does not reveal the target, the probability mass $\rmap(i,j)$ of the current cell $(i,j)$ is then reduced to $0$\footnote{More complex models specify a probability of detection (POD) given that robot and target are in the same area \cite{cooper2003compatibility}. For now, our work assumes a probability of detection (POD) of 1. A more realistic POD could easily be included in our approach by updating the probability mass in the cell accordingly. 

Furthermore, note that as probability mass is cleared, the numbers in $\rmap$ no longer sum up to 1, so $\rmap$ is an unnormalized probability distribution.}.

The most intuitive definition of the reward function corresponding to (Eq. \ref{eq:objective}) would give the reward of $1$ for recovering the target, coupled with a discount factor $\gamma=\exp(-1/c)$. However, this reward function has a high variance, as with the exact same search strategy the target could be found quickly, slowly, or not at all due to chance. Instead, we reward the robot for scanning cells with a high probability of containing the target. This does not introduce bias in the policy optimization, while reducing statistical variance\footnote{Proofs are given in the Appendix.}. A lower statistical variance typically allows optimal policies to be learned using fewer sampled trajectories.

%Our current task of exploring the region of interest for a lost target requires a spatial coverage such that the regions with higher probability are visited earlier. The state space $S$ comprises of robot's current position ($x_{ij}$) and the current rewardmap ($R_{\textnormal{map}}$). The robot selects its action $a$ according to a control policy $\pi$. In this paper we consider a stochastic parameterized control policy $\pi(a|s)$. We consider a deterministic transition function $s_{t+1}=T(s_t,a_t)$ and the states and actions of the robot jointly form a trajectory $\tau=(s_0,a_0,s_1,a_1,...)$. The reward gained for every action $a_t$ in state $s_t$ is represented by $r_t(s_t,a_t)$ and the accumulated discounted reward for a trajectory $\tau$ of finite horizon $H$ is given by $$R(\tau) = \sum_{t=0}^{H} \gamma^{t} r_t(s_t,a_t),$$ where $\gamma \in [0,1)$ is a discount factor that discounts rewards further in the future.

%\textcolor{green}{question for Sandeep - is this accurate? }
%Sandeep - Agree

\begin{figure*}[hb]
\centering
\begin{subfigure}[t]{0.45\textwidth}
		\captionsetup{justification=centering}
        \includegraphics[width=\textwidth,height=5cm]{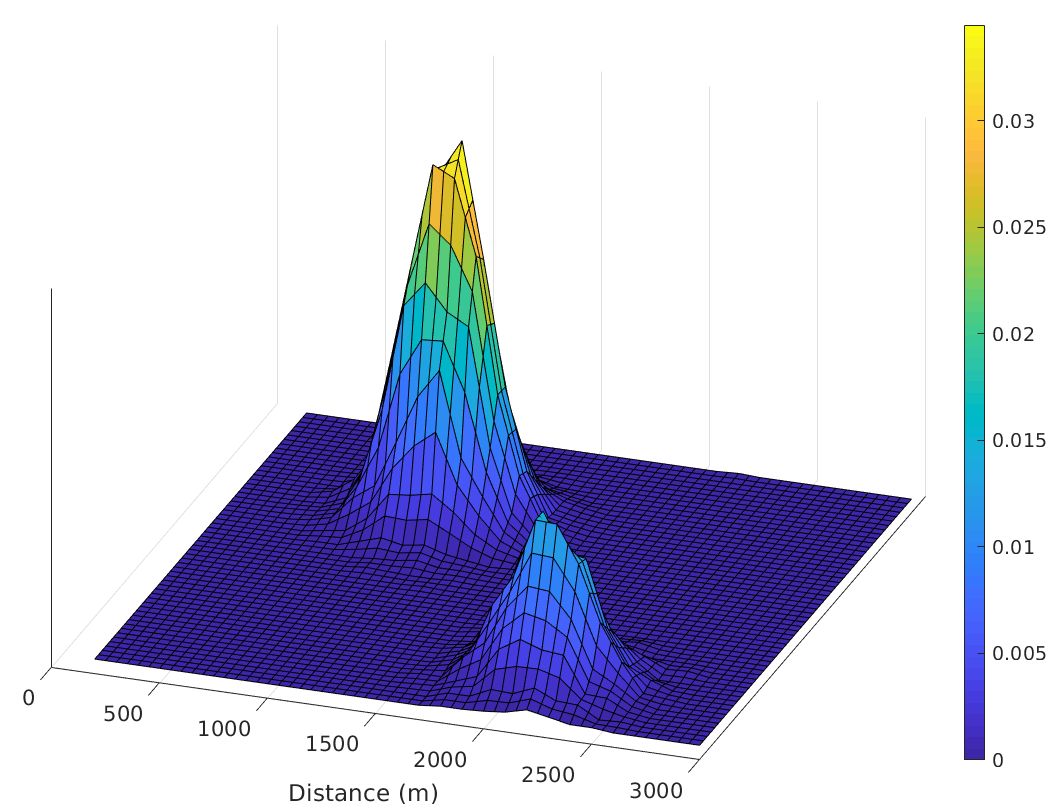}
        \caption{test-scenario1 - mixture of Gaussians based probability distribution}
        \label{fig:test1}
    \end{subfigure}
       \begin{subfigure}[t]{0.45\textwidth}
       \captionsetup{justification=centering}
        \includegraphics[width=\textwidth,height=5cm]{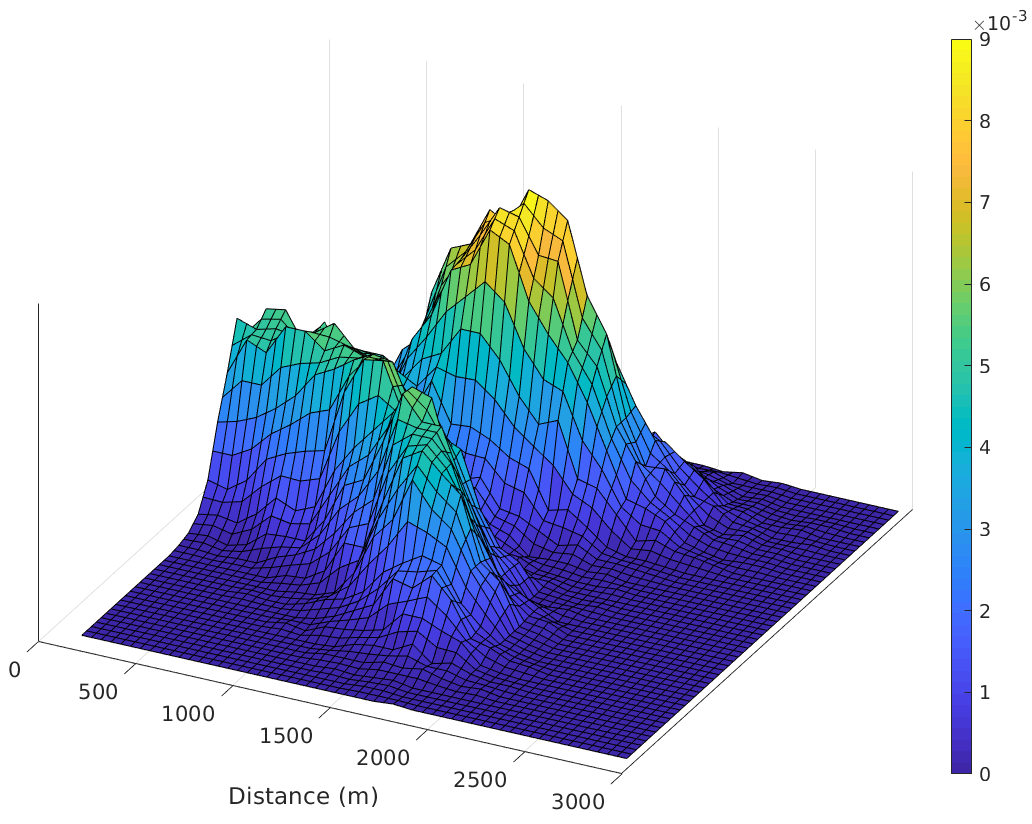}
        \caption{test-scenario2 - non-gaussian probability distribution}
        \label{fig:test2}
    \end{subfigure}
    \caption{Probability distributions used for testing. Colorbars indicate the probability of finding the lost target.}
    \label{fig:results}
    \vspace{-1em}
\end{figure*}

\subsection{Policy Gradient Approach}
Policy gradient methods use gradient ascent for maximizing the expected return $J_\theta$. The gradient of the expected return ($\nabla_\theta J_\theta$) guides the direction of the parameter ($\theta$) update. The policy gradient update is given by,
$$
\theta_{k+1} = \theta_k + \eta \nabla_\theta J_\theta,
$$
where  $\eta$ is the learning rate. The likelihood ratio policy gradient~\cite{williams1992simple} is given by, 
\begin{equation}
\label{eq:update}
\nabla_\theta J_\theta = \int_\tau \nabla_\theta p_\theta(\tau) R(\tau) d\tau
\end{equation}
However, this expression depends on the correlation between actions and previous rewards. These terms are $0$ in expectation, but cause additional variance. Ignoring these terms yields lower-variance updates, which are used in the Policy Gradient Theorem (PGT) algorithm and the GPOMDP algorithm
\cite{baxter2001infinite, sutton2000policy,deisenroth2013survey}. Accordingly, the policy gradient is given by,
%\begin{equation}
\begin{multline}
\label{eq:gradient}
\nabla_\theta J_\theta = \frac{1}{m}\sum_{i=1}^m\sum_{t=0}^{H-1}\nabla_\theta \log \pi_\theta(a_t^{(i)}|s_t^{(i)}) \\
\left(\sum_{j=t}^{H-1}r(s_j^{(i)},a_j^{(i)}) - b(s_t^{(i)})\right).
\end{multline}
%\end{equation}
In this equation, the gradient is based on $m$ sampled trajectories from the system, with $\vec{s}_j^{(i)}$ the state at the $j^{\text{th}}$ time-step of the $i^{\text{th}}$ sampled roll-outs. Furthermore, $b$ is a variance-reducing baseline. In our experiments, we set it to the observed average reward. 

% \begin{figure}[h]
% \centering
% \captionsetup{justification=centering}
% \includegraphics[width=0.35\textwidth]{imgs/features1.png}
% \caption{Feature design 1 - centered around the robot with a $9X9$  $R_{map}$. This is all-grid uniform feature design.}
% %\vspace{-1.5em}
% \label{fig:features1}
% \end{figure}

% \begin{figure}[h]
% \centering
% \captionsetup{justification=centering}
% \includegraphics[width=0.35\textwidth]{imgs/features2.png}
% \caption{Feature design 2 - centered around the robot with $nXn$ $R_{map}$, with $n \ge 10$. This is a multi-resolution feature design.}
% %\vspace{-1.5em}
% \label{fig:features2}
% \end{figure}

\subsection{Policy design}
We consider a policy that is Gibbs distribution in a linear combination of features given by
\begin{equation}
\label{eq:policy}
\pi(\mathbf{s},\mathbf{a}) = \frac{e^{\theta^T\phi_{sa}}}{\sum_be^{\theta^T\phi_{sb}}},\hspace{1cm}\forall s \in S; a,b \in A,
%k(\mathbf{x}, \mathbf{x}') = \sigma_f^2 \exp\Big(-\frac{|\mathbf{x} - \mathbf{x}' |^2}{2\mathfrak{l}^2}\Big),
\end{equation}
where $\phi_{sa}$ is an $l$-dimensional feature vector characterizing state-action pair ($s,a$) and $\theta$ is an $l$-dimensional parameter vector. This is a commonly used policy in reinforcement learning approaches \cite{sutton2000policy}.
%The feature vector $\phi_{sa}$ is a mapping between the feature representation $\phi'$ of the state space and the action space, thus $l = |\phi'|\times|A|$. All elements of $\phi_{sa}$ are $0$ except for the ones mapping to current action $a$.
The final feature vector $\phi_{sa}$ is formed by concatenating a vector $\phi'_{s} \delta_{a a'}$ for every action $a' \in \{North, East, South, West\}$, where $\phi'_{s} \subset \mathbb{R}^k$ is a feature representation of the state space, and $\delta_{a a'}$ is the Kronecker delta. Thus, the final feature vector has $4 \times k$ entries, $75\%$ of which corresponding to non-chosen actions will be $0$ at any one time step.

We consider two types of feature representations ($\phi'_{s}$) for our approach. The first feature construction is to consider a vector that represents all rewards in $\rmap$ robot-centric, as illustrated in Fig. \ref{fig:features1}. This feature vector grows in length as the size of the search region increases, resulting in higher computation times for bigger regions. In the second design we consider a multi-resolution feature aggregation resulting in a fixed number ($24$) of features irrespective of the size of the search region. In this case, features corresponding to larger cells are assigned a value equal to the average of values of $\rmap(i,j)$ that fall in that cell. In multi-resolution aggregation, the feature cells grow in size along with the distance from the robot location as depicted in Fig. \ref{fig:features2}. Thus, areas close to the robot are represented with high resolution and areas further from the robot are represented in lower resolution. The intuition behind this feature design is that the location of nearby target probabilities is important to know exactly, while the location of faraway target probabilities can be represented more coarsely. The multi-resolution feature design is also suitable for bigger worlds as it scales logarithmically to the size of the world.

Both these feature designs were tested with the policy gradient based searching algorithm and we found that the multi-resolution aggregated features produce higher accumulated rewards and better discounted rewards (Fig. \ref{fig:total}, \ref{fig:discounted}). Also, the computation gets quadratically expensive as the size of the grid world increases (Fig. \ref{fig:cpu_time}). Based on these results, a non-uniform, multi-resolution, robot-centric feature design is more beneficial for efficient searching. These results further strengthen our belief that the immediate actions are influenced by the nearer rewards and the farther low-resolution features enhance non-myopic planning of the whole path. Theoretically, replacing the individual cell values by their averages causes some information loss. However, to plan for far-away target probabilities, a coarse location is enough. In fact, the results in Fig. \ref{fig:total} and Fig. \ref{fig:discounted} show that, perhaps due to the regulating effect of averaging, multi-resolution grids perform better than the uniform representation even after extensive training.

Further in this paper we will only consider the multi-resolution robot-centric feature design in our policy gradient search algorithm. The aggregated feature design is only used to achieve better policy search, but the robot action is still defined at the grid-cell level.

\section{EXPERIMENTAL RESULTS AND DISCUSSION}
In this section, we will introduce the experimental set-up and baseline methods, before presenting and discussing the experimental results.

\begin{figure*}[h]
\centering
    \begin{subfigure}[t]{0.30\textwidth}
    \captionsetup{justification=centering}
        \includegraphics[width=\textwidth,height=5cm]{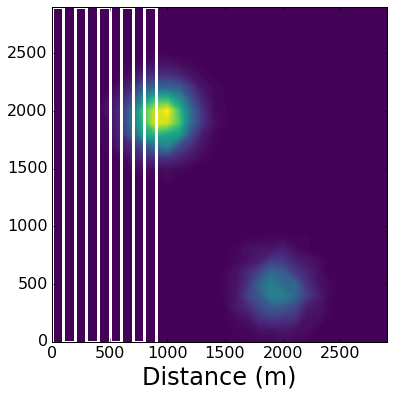}
        \caption{Boustrophedonic coverage \\over test scenario 1}
        \label{fig:lawn_test1}
    \end{subfigure}
    \begin{subfigure}[t]{0.27\textwidth}
    \captionsetup{justification=centering}
        \includegraphics[width=\textwidth,height=5cm]{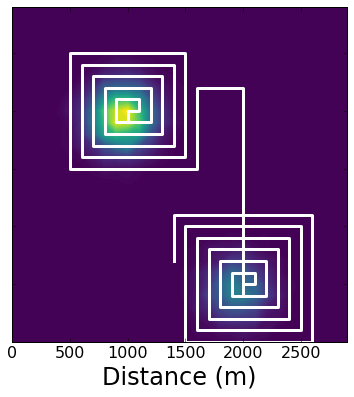}
        \caption{Informed spiral search on test scenario 1}
        \label{fig:spiral_test1}
    \end{subfigure}
    \begin{subfigure}[t]{0.31\textwidth}
    \captionsetup{justification=centering}
        \includegraphics[width=\textwidth,height=5cm]{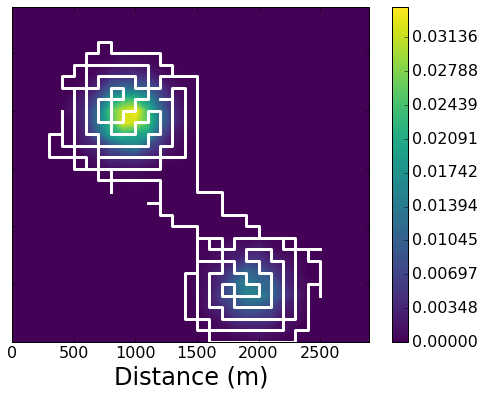}
        \caption{Action plan generated by our method for test scenario 1}
        \label{fig:pg_test1}
    \end{subfigure}
    \begin{subfigure}[t]{0.30\textwidth}
    \captionsetup{justification=centering}
        \includegraphics[width=\textwidth,height=5cm]{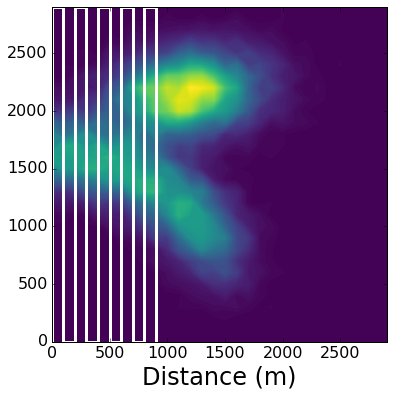}
        \caption{Boustrophedonic coverage \\over the test scenario 2}
        \label{fig:lawn_test2}
    \end{subfigure}
    \begin{subfigure}[t]{0.27\textwidth}
    \captionsetup{justification=centering}
        \includegraphics[width=\textwidth,height=5cm]{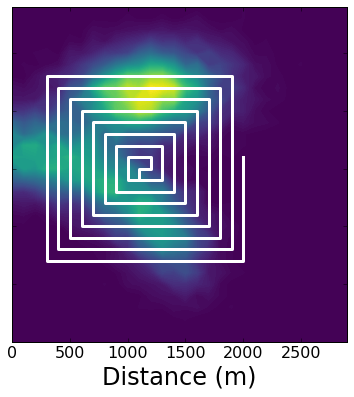}
        \caption{Informed spiral search on test scenario 2}
        \label{fig:spiral_test2}
    \end{subfigure}
    \begin{subfigure}[t]{0.31\textwidth}
    \captionsetup{justification=centering}
        \includegraphics[width=\textwidth,height=5cm]{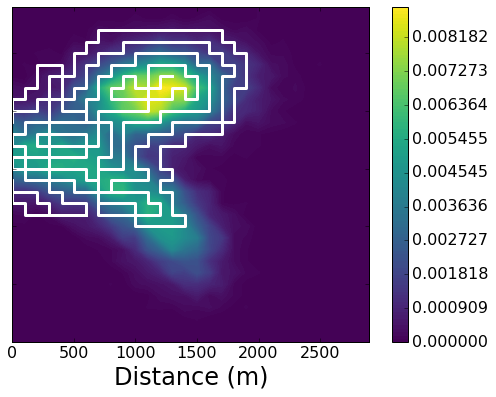}
        \caption{Action plan generated by our method for test scenario 2}
        \label{fig:pg_test2}
    \end{subfigure}
    \begin{subfigure}[t]{0.48\textwidth}
    \captionsetup{justification=centering}
        \includegraphics[width=\textwidth,height=5.5cm]{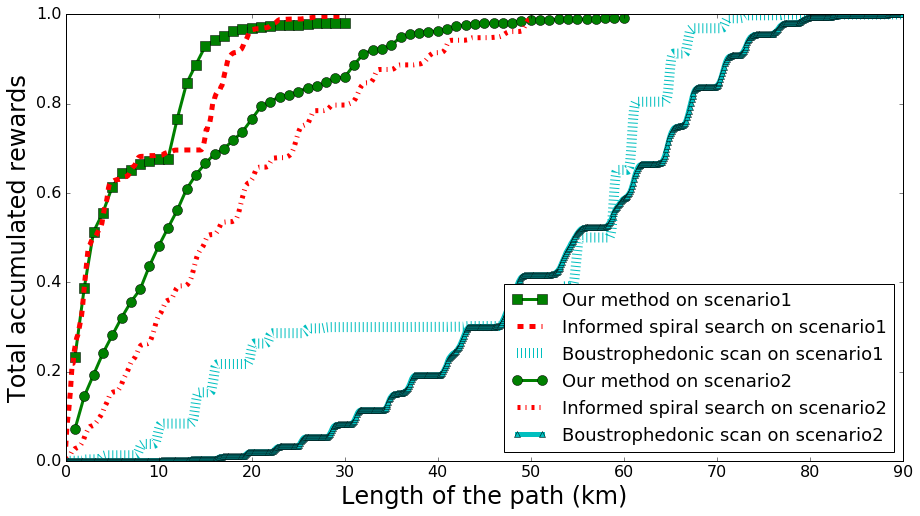}
        \caption{Total accumulated rewards vs. length of the path}
        \label{fig:total_reward}
    \end{subfigure}
    \begin{subfigure}[t]{0.48\textwidth}
    \captionsetup{justification=centering}
        \includegraphics[width=\textwidth,height=5.5cm]{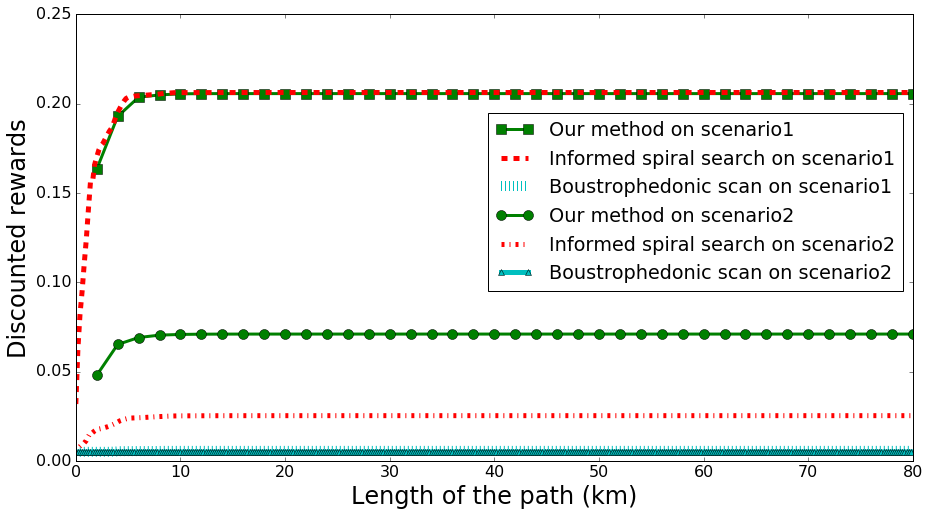}
        \caption{Discounted reward vs. length of the path. For our experiments, we used a discount factor $\gamma = 0.9$}
        \label{fig:disc_reward}
    \end{subfigure}
    \captionsetup{justification=centering}
    \caption{Comparison of three search methods. The trajectories shown are from \\running the three search methods for $300$ time steps.}
    \label{fig:results}
    \vspace{-2em}
\end{figure*}

\subsection{Setup}
We generated a generic training scenario with probability distributions for a lost target using Gaussian mixtures. We model the search space of a simulated aerial search vehicle as a $30\times30$ grid world with each grid-cell spanning $100m\times100m$. We used 20 roll-outs in every iteration of the training phase. A discount factor of $0.9$ was used in these experiments. During the test phase, an action with maximum probability is chosen at a given state. Fig. \ref{fig:model} illustrates the probability distribution grid used for training our policy based searcher. However, as mentioned in Section \ref{sec:model}, the searcher algorithm could be trained on any other type of distribution too. Two significantly different test-scenarios are presented to evaluate and compare the search algorithms. The first test-scenario (Fig. \ref{fig:test1}) comprises of two Gaussians imitating the probability distribution of a lost target. The second test-scenario (Fig. \ref{fig:test2}) cannot be represented as a mixture of a few Gaussians.

\subsection{Baselines}

Traditionally, exhaustive sampling of a partially observable, obstacle-free region employs a boustrophedon path~\cite{Choset97coveragepath}. The \emph{boustrophedon} or \emph{lawnmower path} is the approach a farmer takes when using an ox to plow a field, making back and forth straight passes over the region in alternating directions until the area is fully covered.
Another efficient search pattern reported in the robotic search literature is spiral, which minimizes the time to find a single stationary lost target in several impressive idealized scenarios \cite{burlington1999spiral,meghjani2016multi}.

We compare our search algorithm with these two search techniques. Fig. \ref{fig:lawn_test1}-\ref{fig:pg_test2} illustrate the paths generated by the three coverage algorithms overlaid on top of the test scenario probabilities. The paths shown are generated by running the search algorithms for $300$ time-steps.

\subsection{Results and discussion}

Spiral search can generate efficient, or even optimal, paths under suitable conditions such as the \emph{first test scenario} (Fig. \ref{fig:spiral_test1}). Its satisfactory performance is only assured for a restricted class of unimodal distributions (as opposed to that in Fig. \ref{fig:spiral_test2}). We compare these algorithms based on the rewards collected by removing the probability mass in the search region, corresponding to the (discounted) probability of finding the target. Our goal is to reduce the probability mass as fast as possible by visiting regions with high probability mass (\emph{hotspots}). We use discounted rewards as a metric to measure the timeliness of an algorithm.

The plots in Fig. \ref{fig:total_reward} illustrate how our proposed method is the fastest to cover the target probability mass on both test scenarios, even though all methods eventually visit all states of potential interest. Spiral search performs on par with our algorithm in terms of total rewards accumulated. Nonetheless, policy based searcher out performs both, spiral and boustrophedon search techniques in total discounted rewards on the \emph{second test scenario} by a significant margin. Thus the policy based searcher exhibit a higher order of timeliness, which is a key requirement for any search algorithm to be applicable in search and rescue missions.

It is important to note that despite the dissimilarity between the training scenario and the test scenario, the policy-based search algorithm achieves better performance than the other methods.
%An important thing to notice is the dissimilarity between the \emph{test scenario 2} and the training scenario. Still the policy based search algorithm is capable of performing higher rewarding actions.

\section{CONCLUSIONS}
We presented an optimization algorithm that results in a policy that can generate non-myopic search plans for novel environments in real time. Our policy gradient based search algorithm is well suited for applications in search and rescue because of its ability to come up with a search plan on-the-go, based on the new probability distribution of the lost target, with no time wasted on retraining. Other time-critical applications like aerial surveys after a calamity, water surface surveys to monitor and contain algal blooms, and searching for remains under ocean after an accident, can use the presented algorithm to explore the region of interest.

In the near future, we would like to enhance the performance of the search algorithm by exploring different feature aggregation techniques and designing an adaptive aggregation that can combine features in real-time according to the observations made during the survey.

\section*{ACKNOWLEDGMENT}

We would like to thank the the Natural Sciences and Engineering Research Council (NSERC) for their support through the Canadian Field Robotics Network (NCFRN).

%%%%%%%%%%%%%%%%%%%%%%%%%%%%%%%%%%%%%%%%%%%%%%%%%%%%%%%%%%%%%%%%%%%%%%%%%%%%%%%%

%References are important to the reader; therefore, each citation must be complete and correct. If at all possible, references should be commonly available publications.

\appendix
First, we want to show that giving a reward for clearing probability mass instead of locating the target does not bias the objective function.

\begin{prop}

\[
  \mathbb{E}_{\tau_{\vec{\theta}}}\left[\sum_{t=0}^{H-1} \gamma^t r(\vec{s}_t,\vec{a}_t) \right] = \mathbb{E}_{\tau_{\vec{\theta}}, \targetloc}[ R(\tau)]   
\]
where the discount factor $\gamma=\exp(-1/c)$, $\targetloc$ is the target's location, and the proxy reward $r(\vec{s}_t, \vec{a}_t)$ is equal to the probability of the target being at the robot's location $\robotloc$ according to $\rmap$\footnote{The reward map $\rmap$ is not normalized, yet after clearing a fraction $g$ or probability mass, the probability that the target has not been found yet is $1-g$. If the target were not found yet, the normalized probability that the target is in cell $(i,j)$ is $q(i,j)/(1-g)$. So the probability of finding the target in $(i,j)$ indeed equals $(1-g)q(i,j)/(1-g)=q(i,j)$.}. $R(\tau) = \exp(-T/c)$ with $T$ is the time until the target is found if the target is found within $H$ time steps, or $0$ otherwise\footnote{The proof here assumes a static target for notational simplicity, but can be generalized to dynamic targets by  making both $r$ and $\targetloc$ time-step dependent, and calculating expected values over all $\targetloc_1,\ldots, \targetloc_H$ jointly.}. 
\label{prop:expectation}
\end{prop}

\begin{proof}
For any trajectory, 
\begin{align*}
\mathbb{E}_{\targetloc} R(\tau) &= \begin{cases}\mathbb{E}_{\targetloc} \left[\exp{(-1/c)}^T\right] & \textnormal{if } T\le H \\ 0 & \textnormal{otherwise}\end{cases}\\
&= \mathbb{E}_{\targetloc} \left[\sum_{t=0}^{H-1} \gamma^t \delta(\robotloc_t, \targetloc)\right]  \\
&=  \sum_{t=0}^{H-1} \gamma^t \mathbb{E}_{\targetloc} [\delta(\robotloc_t, \targetloc)] = \sum_{t=0}^{H-1} \gamma^t r(\vec{s}_t,\vec{a}_t),
\end{align*}
where $\delta$ is the Kronecker delta. Since this equality holds for any trajectory, it must hold for a linear combination of trajectories.
\end{proof}

Now, we want to show that the variance of gradient estimates using the proxy reward for clearing probability mass is lower than that of the gradient estimates using the original objectives.
\begin{prop}
\begin{multline*}
\Var \sum_{t=0}^{H-1}\nabla_\theta \log \pi_{\vec{\theta}}(\vec{a}_t|\vec{s}_t) 
\sum_{j=t}^{H-1} \gamma^j r(\vec{s}_j,\vec{a}_j) \\
\le 
\Var \mathbb{E}_X\sum_{t=0}^{H-1}\nabla_\theta \log \pi_\theta(\vec{a}_t|\vec{s}_t) 
\sum_{j=t}^{H-1} \gamma^j \delta(\robotloc_j, \targetloc)
\end{multline*}
\end{prop}

\begin{proof}
We use the definition of the variance and rearrange terms as in the GPOMDP method \cite{baxter2001infinite} to reformulate the proposition. We introduce the shorthand $r_j = r(\vec{s}_j, \vec{a}_j)$ make the equations more readable, and obtain
\begin{align}
&\mathbb{E}_\tau \left[ \left( \sum_{j=0}^{H-1} \gamma^j r_j \sum_{t=0}^j \nabla_{\vec{\theta}} \log \pi_{\vec{\theta}}(\vec{a}_t|\vec{s}_t) \right)^2 \right] \label{eq:expsqproxy}\\
& - \mathbb{E}_\tau \left[  \sum_{j=0}^{H-1} \gamma^j r_j \sum_{t=0}^j \nabla_{\vec{\theta}} \log \pi_{\vec{\theta}}(\vec{a}_t|\vec{s}_t)  \right]^2 \label{eq:sqexpproxy} \\
\le \phantom{.} &\mathbb{E}_\tau \left[ \left( \mathbb{E}_{\targetloc} \sum_{j=0}^{H-1} \gamma^j \delta(\robotloc_j, \targetloc) \sum_{t=0}^j \nabla_{\vec{\theta}} \log \pi_{\vec{\theta}}(\vec{a}_t|\vec{s}_t) \right)^2 \right] \label{eq:expsqobj}\\
& - \mathbb{E}_\tau \left[ \mathbb{E}_X \sum_{j=0}^{H-1} \gamma^j \delta(\robotloc_j, \targetloc) \sum_{t=0}^j \nabla_{\vec{\theta}} \log \pi_{\vec{\theta}}(\vec{a}_t|\vec{s}_t)  \right]^2. \label{eq:sqexpobj}
\end{align}
Note that \eqref{eq:sqexpproxy} and \eqref{eq:sqexpobj} are expectations of unbiased estimates of the respective objective, and so equal  the   gradient of the expected value of the respective objective.  By Proposition~\ref{prop:expectation}, these expected values are the same, so \eqref{eq:sqexpproxy} and \eqref{eq:sqexpobj} cancel each other out. Writing out the squares in \eqref{eq:expsqproxy} and \eqref{eq:expsqobj} and combining like terms yields
\begin{align*}
&\mathbb{E}_\tau \left[ \sum_{j=0}^{H-1} \sum_{i=0}^{H-1}   \gamma^{i+j} \left( r_j  r_i  -  \mathbb{E}_{\targetloc} \left[ \delta(\robotloc_i, \targetloc)  \delta(\robotloc_j, \targetloc)  \right]\right)  S_j S_i \right] \le 0
%\\&=\sum_{j=0}^{H-1} \sum_{i=0}^{H-1}   \gamma^{i+j} \left( \mathbb{E}_\tau\left[ r_j  r_i S_j S_i\right]  -  \mathbb{E}_\tau\left[\delta(s_j^{\textnormal(robot)}, X)  \delta(s_i^{\textnormal(robot)}, X) S_j S_i \right] \right)   \le 0,	
\end{align*}
where we introduced the shorthand
\[
S_j = \sum_{t=0}^j \nabla_{\vec{\theta}} \log \pi_{\vec{\theta}}(\vec{a}_t|\vec{s}_t),
\]
and used the fact that $S_j$, $S_i$ are independent of the target location $\targetloc$. Note that $r_i$ is just the probability that $\delta(\robotloc_i, \targetloc)=1$,  and since the agent cannot find the target twice, $\delta(\robotloc_i, \targetloc)\delta(\robotloc_j, \targetloc)$ is non-zero only if $i=j$. Thus, the proposition is equivalent to
\begin{align}
\mathbb{E}_\tau \left[ \sum_{j=0}^{H-1} \sum_{i=0}^{H-1} \gamma^{i+j} \left(r_j  r_i  -  \delta(i,j) r_i \right)  S_j S_i \right] \le 0. \label{eq:propequivalent}
\end{align}
The left-hand side of this inequality can be expressed and upper-bounded as follows:
\begin{align}
& \mathbb{E}_\tau \left[ \sum_{j=0}^{H-1} \sum_{i=0}^{H-1} \gamma^{i+j} \left(r_j  r_i  -  \delta(i,j) r_i \right)  S_j S_i \right] \label{eq:propequivalent2}\\
& =\sum_{i=0}^{H-1} \sum_{j=0}^{H-1} \gamma^{i+j} r_i r_j S_i S_j -\sum_{i=0}^{H-1} r_i \gamma^{2i} S_i^2 \nonumber \\
&\le \left(2-\sum_{k=0}^{H-1}r_k\right) \sum_{i=0}^{H-1} \sum_{j=0}^{H-1} \gamma^{i+j} r_i r_j S_i S_j -\sum_{i=0}^{H-1} r_i \gamma^{2i} S_i^2  \nonumber \\
&=-\mathbb{E}_\tau \left[ \sum_{i=0}^{H-1} r_i \left(\gamma^{i} S_i -  \sum_{j=0}^{H-1} \gamma^{j} r_j S_j\right)^2 \right]. \label{eq:upperbound}
\end{align}
The inequality is due to the sum of rewards always being smaller than 1 (as the rewards denote probabilities of finding the target, and would sum up to 1 if and only if the robot visits all cells where probability mass was initially present in this trajectory). Note that the expected value in \eqref{eq:upperbound} is non-negative: it is a weighted sum of squared terms, where all weights are non-negative (again, due to their interpretation as probabilities). Thus, \eqref{eq:upperbound} is non-positive, so  \eqref{eq:propequivalent2} must $\le 0$, confirming \eqref{eq:propequivalent}.
\end{proof}

\bibliographystyle{IEEEtran}
\balance
\bibliography{root,refs}

\end{document}